\documentclass[12pt]{article}
\pdfoutput=1
\usepackage{amsmath,amsbsy,amssymb,amsthm}
\usepackage{url}
\usepackage{color}
\usepackage{natbib}
\usepackage{graphicx}
\usepackage{thm-restate}
\usepackage{algorithm}
\usepackage{algpseudocode}
\usepackage[]{hyperref}
\usepackage{hyphenat}
\algnewcommand\algorithmicinput{\textbf{Input:}}
\algnewcommand\algorithmicoutput{\textbf{Output:}}
\algnewcommand\Input{\item[\algorithmicinput]}
\algnewcommand\Output{\item[\algorithmicoutput]}
\usepackage[margin=1in]{geometry}

\usepackage{booktabs,tabularx}
\usepackage{multirow}

\usepackage{setspace}
\newtheorem{Theorem}{Theorem}[section]
\newtheorem{Definition}[Theorem]{Definition}

\newtheorem{Assumption}[Theorem]{Assumption}

\newtheorem{Lemma}[Theorem]{Lemma}

\numberwithin{equation}{section}

\newcommand{\h}{\hspace*{.24in}}

\newcommand{\mc}{\mathcal}

\newcommand{\argmin}{\operatornamewithlimits{argmin}}

\title{\textnormal{Consistent feature selection for neural networks via Adaptive Group Lasso}}

\author{
Vu Dinh\thanks{These authors contributed equally to this work.} \\
Department of Mathematical Sciences \\
University of Delaware \\
Newark, Delaware, USA
\and
Lam Si Tung Ho$^{*}$ \\
Department of Mathematics and Statistics \\
Dalhousie University \\
Halifax, Nova Scotia, Canada
}

\date{}

\begin{document}

\maketitle

\clearpage

\begin{abstract}
One main obstacle for the wide use of deep learning in medical and engineering sciences is its interpretability. While neural network models are strong tools for making predictions, they often provide little information about which features play significant roles in influencing the prediction accuracy. To overcome this issue, many regularization procedures for learning with neural networks have been proposed for dropping non-significant features. Unfortunately, the lack of theoretical results casts doubt on the applicability of such pipelines.
In this work, we propose and establish a theoretical guarantee for the use of the adaptive group lasso for selecting important features of neural networks. Specifically, we show that our feature selection method is consistent for single-output feed-forward neural networks with one hidden layer and hyperbolic tangent activation function. We demonstrate its applicability using both simulation and data analysis.
\end{abstract}


\section{Introduction}
\label{intro}

One main obstacle for the wide use of deep learning in medical and engineering sciences is its interpretability and explainability. 
Due to the complicated nature of deep neural networks and deep learning methods, they have been mostly treated as black-box tools.
Although the model has strong predictive power, it often provides little information about which features play significant roles in influencing prediction accuracy. 
This limitation creates a bottle-neck for the applicability of deep learning in a time when the interpretability of scientific results is becoming increasingly important. 
This issue is even more severe in application contexts where the most important aspect of the learning problem may not be about prediction but about 
feature selection (that is, detecting which inputs control the outputs).
A prime example of such scenarios is the task of identifying genes that increase the risk of cancer.

 Recently, a wide variety of methods have been proposed to perform feature selection for neural networks. 
 A popular approach is to derive a ranking of feature importance based on local perturbations.
This includes fitting a model in the local region around the input or locally perturbing the input to see how the prediction changes \cite{simonyan2013deep, ribeiro2016should, lundberg2017unified, shrikumar2017learning, ching2018opportunities}.
For high-dimensional inputs, high-level representation of the data can also be extracted, for example by an auto-encoder \cite{nezhad2016safs} or a Deep Boltzmann Machine \cite{ibrahim2014multi, taherkhani2018deep}) before a classical feature selection method can be applied. 
Although these methods can provide useful insights, they often focus on specific architectures of the neural nets and can be difficult to generalize \cite{lu2018deeppink}. 

In another direction, researchers employ regularization procedures to drop non-significant features out of the model.                                                                                                           
Regularization methods have been used extensively for reducing the number of parameters in a deep neural network and 
can be naturally adapted to obtain heuristic methods for feature selection by penalizing the weights of the first hidden layer. 
Lasso ($\ell_1$ regularization), which penalizes the sum of absolute values of the weights, is probably the most popular regularization method.
This method, however, is not 
ideal for feature selection because a feature can only be dropped if all of its connections have been shrunk to zero together, an objective that is not actively pursued in the scope of Lasso. 
\citet{zhao2015heterogeneous} and \citet{scardapane2017group} address this concern by utilizing Group Lasso for selecting features of deep neural networks.
Alternatively, \citet{li2016deep} propose adding a sparse one-to-one linear layer between the input layer and the first hidden layer of a neural network and performing $\ell_1$ regularization on the weights of this extra layer.


While pipelines for feature selection for neural networks exist, the lack of theoretical results casts doubt on their applicability to real-world data.
Even worse, several works have indicated that Lasso and Group Lasso could be inconsistent for feature selection \cite{zou2006adaptive, zhang2018non}.
This is further complicated by the fact that neural network models are highly non-linear. 

We tackle this issue directly by proposing and establishing a theoretical guarantee for the use of the Adaptive Group Lasso for selecting important features of neural networks. 
Using the empirical risk minimizer (ERM) or the Group Lasso as an initial estimate,
the Adaptive Group Lasso constructs a data-dependent weighted regularizing function in such a way that as the sample size grows, the penalty for non-significant features get inflated (to infinity), whereas the penalty for significant features are bounded. 
Under this framework, we show that for a single-output feed-forward neural network model with one hidden layer and the hyperbolic tangent activation function, the proposed feature selection method is consistent.
Additionally, we demonstrate the performance of our method in both simulation and data analysis.

\section{Mathematical framework}
\label{sec:math}

\subsection{Generating model}
 
For the simplicity of presentation, we consider a single-output feed-forward neural network for regression with one hidden layer and hyperbolic tangent activation function, where the corresponding numbers of nodes in each layer are $(n_I, n_H, 1)$.
We also separate the inputs into two groups $ s  \in \mathbb{R}^{n_s}$ and $ z  \in \mathbb{R}^{n_z}$ (with $n_s + n_z = n_I$) that denote the significant and non-significant variables, respectively.
An input is non-significant when the true value of all weights associated with it is zero.
We note that this separation is simply for mathematical convenience and the training algorithm is not aware of such dichotomy. 

The forward model is visualized in Figure \ref{fig:network} and can be summarized as
\[
f_{u, v, w, b_1, b_2}( s,  z ) = w \cdot h + b_2
\]
where
\[
h^{[i]} = \tanh \left (\sum_{k=1}^{n_s}{u^{[i, k]} ~s^{[k]}} + \sum_{k=1}^{n_z}{v^{[i, k]} ~z^{[k]}} + b_1^{[i]} \right).
\]
Here, $u \in \mathbb{R}^{n_H \times n_s}$, $v \in \mathbb{R}^{n_H \times n_z}$, $h,w, b_1 \in \mathbb{R}^{n_H}$, $b_2 \in \mathbb{R}$; $h^{[i]}$ denotes the $i$-th component of a vector $h$, and $u^{[i, k]}$ denotes the $[i,k]$-entry of a matrix $u$.

We will study the learning problem in the model-based setting, whereas training data $\{(X_i, Y_i)\}_{i=1}^n$ are independent and identically distributed (i.i.d ) samples generated from $P^*_{X, Y}$ such that
\[
Y_i= f_{u^*, 0, w^*, b_1^*, b_2^*}(X_i) + \epsilon_i
\]
where $\epsilon_i \sim \mathcal{N}(0, \sigma_e^2)$.
Without loss of generality, we assume further that the input density $p_X$ is positive and continuous on its bounded domain and that the weight space (the set of all feasible vectors $\alpha = (u, v, w, b_1, b_2)$ for the model) is a compact set on a Euclidean space.
We note that when $v=0$, the function $f_{u, 0, w, b_1, b_2}$ does not depend on $Z$. 

\begin{figure}
\centering
    \includegraphics[width=0.5\columnwidth]{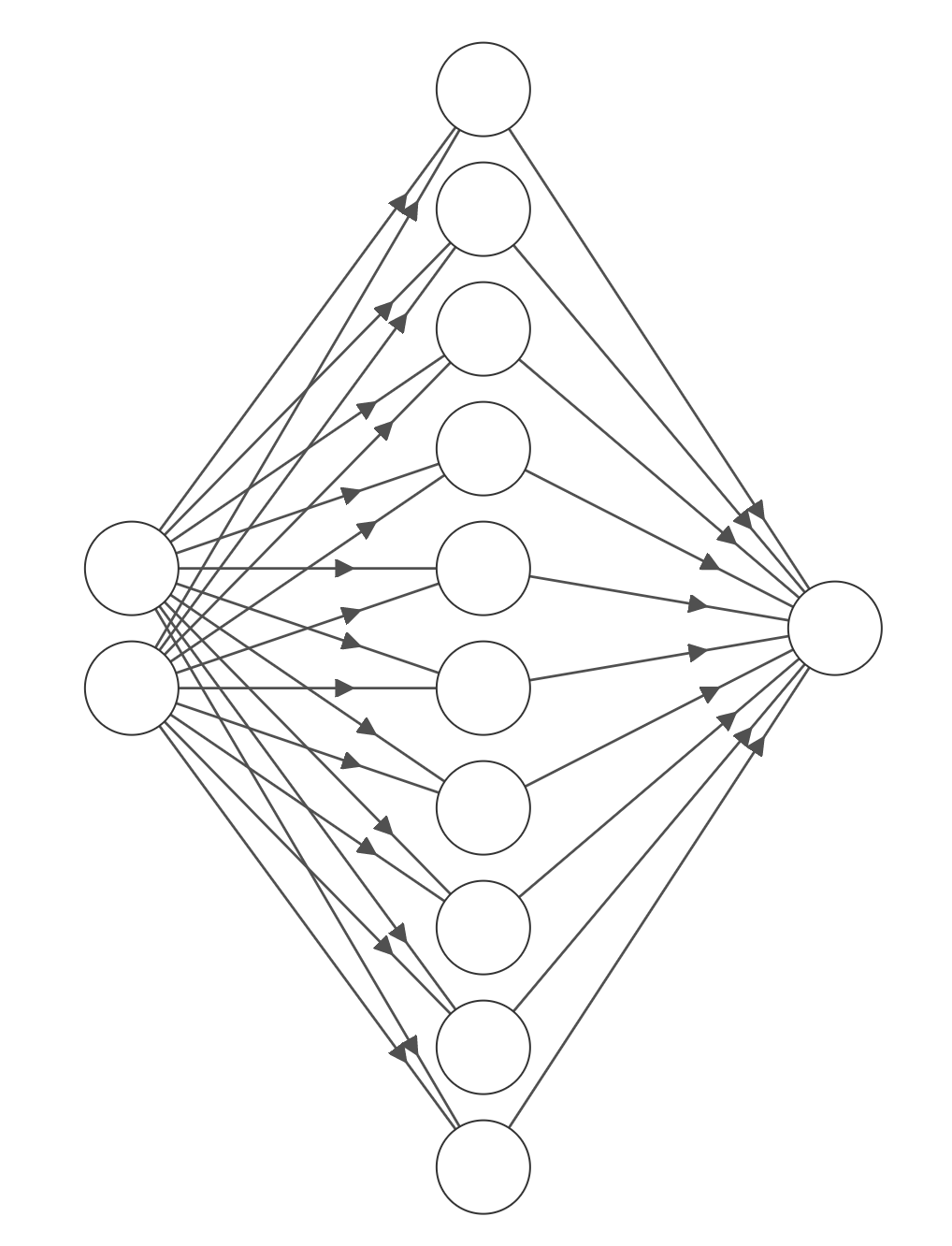}
\caption{A single-output feed-forward neural network for regression with one hidden layer.}
\label{fig:network}
\end{figure}

\subsection{Identifiability of the feed-forward neural network model}

Before moving forward to the description of the feature selection problem, it is worth reminding that one of the distinguishing properties of neural network models is that it is inherently \emph{unidentifiable}. 
A simple rearrangement of the nodes in the same hidden layer, an addition of a constant node or a sign flip of the weights may lead to a new configuration that produces the same input-output mapping. 
These redundancies make it more difficult to analyze the performance of these models in both theoretical and practical perspectives. 

To understand the identifiability of the model, we begin by studying the properties of weight space transformations that leave the network input-output mapping unchanged.
We will refer to the set of all vectors $\alpha = (u, v, w, b_1, b_2)$ as the weight space and denote it by $\mathcal{W}$. 

\begin{Definition}
Two weight vectors $\alpha$ and $\alpha'$ are functionally equivalent (denote by $\alpha \sim_F \alpha'$) if 
\[
f_{\alpha} = f_{\alpha'}.
\]
\end{Definition}

We consider two types of transformations: node interchanges and sign flips. 

\begin{itemize}
\item An interchange consists of a permutation in which the weight vectors of two hidden nodes on the same hidden layer are interchanged.
A compensatory interchange of the weights of the nodes in the next layer that receive the inputs from the two interchanged nodes then removes the effect of the exchange in the previous layer.
\begin{Definition}
Two weight vectors $\alpha$ and $\alpha'$ are \emph{interchange equivalent} (denoted by $\alpha \sim_I \alpha' $)
if there exists a bijection $\pi$ on the set of hidden nodes $\mc{E}$ such that $\alpha_i = \alpha'_{\pi(i)}$ for all $i \in \mc{E}$.
Here, $\alpha_i$ is the set of weights associated with node $i$.
\end{Definition}
\item A sign flip is a transformation where the signs of all of the weights of a node on a hidden layer are multiplied by -1. 
A compensatory sign flip is then carried out on all of the weights of nodes of the next layer associated with the input from the sign-flipped node.
\begin{Definition}
Two weight vectors $\alpha$ and $\alpha'$ are \emph{sign-flip equivalent} (denoted by $\alpha \sim_S \alpha' $)
if for all $i \in \mc{E}$, either $\alpha_i = \alpha'_{i}$ or $\alpha'_i = -\alpha_i$.
\end{Definition}
\end{itemize}

We have the following results.

\begin{Theorem}[\citet{kuurkova1994functionally}]
For a single-output feed-forward neural network with one hidden layer and hyperbolic tangent activation function, functionally equivalence are compositions of interchange and sign flip equivalence.
\[
\sim_F = \sim_I \circ \sim_S.
\]
\end{Theorem}

We recall the following two definitions from \citet{sussmann1992uniqueness} and \citet{fukumizu1996regularity}.

\begin{Definition}
A feed-forward model $f_{u, v, w, b_1, b_2}$ is \textbf{minimal} if no network with fewer hidden units have the same input-output map.

A feed-forward model $f_{u, v, w, b_1, b_2}$ is \textbf{irreducible} if
\begin{itemize}
\item[(i)] $(u^{[i,:]}, v^{[i,:]}) \ne 0$ and $w^{[i]} \ne 0$ for all $i$.
\item[(ii)] For any two different indices $i$ and $j$
\[
(u^{[i,:]}, v^{[i,:]},  b_1^{[i]})  \ne \pm (u^{[j,:]}, v^{[j,:]},  b_1^{[j]}).
\]

\end{itemize}
Here, $u^{[i,:]}, v^{[i,:]}$ denote the $i$-th row of the matrices $u, v$ respectively.
\end{Definition}

We note that the two condition $(i)$ and $(ii)$ guarantee that there are no non-participating or constant hidden nodes and no pair of identical nodes in the generating model.
While the two definitions are related, in general, irreducibility is an intrinsic property that can be verified based on the weights of a network alone, while minimality can only be verified in relative to other feasible networks from the model. 
We have the following theorem that relates the two concepts. 

\begin{Theorem}[\citet{sussmann1992uniqueness}]
For single-output feed-forward neural networks with one hidden layer and the hyperbolic tangent activation function, a network is irreducible if and only if it is minimal.
\label{unique}
\end{Theorem}

Throughout the rest of the manuscript, we make the following assumption. 

\begin{Assumption}
The generating forward model $f_{\alpha^*}$, where $\alpha^*=(u^*, 0, w^*, b_1^*, b_2^*)$, is irreducible. 
\label{irr}
\end{Assumption}

This assumption, along with Theorem $\ref{unique}$, provides us a way to characterize the set $\mathcal{H}^*$ of all weight vectors that produce the same input-output map as the generating model as follows. 

\begin{Lemma}
Under Assumption $\ref{irr}$, we have
\begin{itemize}
\item [(i)] If $\alpha \in \mc{W}$ is a weigh vector such that
$f_{\alpha} =f_{\alpha^*}$, then $f_{\alpha}$ is also irreducible. 
\item [(ii)]If $ (u, v, w, b_1, b_2 ) \sim_F  \alpha^*$ then $v =0$ and $u^{[:, k]} \ne 0$ for all $k = 1, \ldots, n_s$.

\item [(iii)] The set $ \mc{H}^* =\{\alpha \in \mc{W}: f_{\alpha} = f_{\alpha^*}\}$ is finite. 
\end{itemize}
\label{lem:characterization}

\end{Lemma}
\begin{proof}
Since $f_{\alpha^*}$ is irreducible, Theorem $\ref{unique}$ implies that it is also minimal. 
Let $\alpha \in \mc{W}$ is a weigh vector such that $f_{\alpha} =f_{\alpha^*}$
Since $\alpha \in \mc{W}$, the network $f_{\alpha}$ have the same number of hidden units as $f_{\alpha^*}$. 
We conclude that $f_{\alpha}$ is minimal and thus is irreducible.  

Finally, since we can only obtain the same input-output map through compositions of node-interchange and sign-flip transformations (for which the number of all possible transformations is finite), $\mc{H}^*$ is also finite. 
Each of those transformations also leaves the zero/non-zero components unchanged, which validates $(ii)$.
\end{proof}


\section{Feature selection for neural networks}

In this section, we propose an Adaptive Group Lasso procedure for feature selection for neural networks.
The procedure consists of two steps:
\paragraph{Step 1:} An initial estimate for the weights is obtained using either the empirical risk minimizer (ERM) or Group Lasso.
\paragraph{Step 2:} Using the initial estimate, we build a weighted penalty for a new shrinkage estimator, which is utilized for deciding which features are significant.

The details of these steps are as follows:

\subsection{Empirical risk minimizer and Group Lasso}

For a set of training data $\{(X_i, Y_i)\}_{i=1}^n$,
the ERM estimate is defined by
\[
\hat \alpha_n = \argmin_{\alpha \in \mc{W}}  \frac{1}{n}\sum_{i=1}^n{\ell(\alpha, X_i, Y_i)},
\]
where $\ell(\alpha, x, y) = (y-f_{\alpha}(x))^2$, and the Group Lasso estimate is
\begin{align*}
(\bar u_n, &\bar v_n, \bar w_n, \bar b_{1,n}, \bar b_{2,n}) := \argmin_{(u,  v, w, b_1, b_2) }{ ~~\frac{1}{n}\sum_{i=1}^n{\ell(\alpha, X_i, Y_i)} + \lambda_n L(u, v) },
\end{align*}
where 
\[
L(u, v) =  \sum_{k=1}^{n_s}{\|u^{[:, k]}\|} +  \sum_{k=1}^{n_z}{\|v^{[:, k]}\|},
\]
$\lambda_n$ is an appropriate regularizing constant and $\|\cdot\|$ denotes the usual Euclidean $\ell_2$-norm. 

We note that in the setting of feature selection, only parameters of the first layer (namely, $u$ and $v$) are regularized.
Unlike the classical Lasso penalty (which takes the summation of the absolute values of individual parameters), parameters that associate with the same input are grouped together (through the $\ell_2$-norm) in the Group Lasso. 
This grouping forces all outgoing connections from a non-significant input neuron to be simultaneously shrunk to zero.

\subsection{Adaptive Group Lasso}

\paragraph{ERM initialization:} We define the ERM + Adaptive Group Lasso estimator as
\begin{align*}
(\tilde u_n, &\tilde v_n, \tilde w_n, \tilde b_{1,n}, \tilde b_{2,n}) := \argmin_{(u,  v, w, b_1, b_2) }{ ~~\frac{1}{n}\sum_{i=1}^n{\ell(\alpha, X_i, Y_i)} + \zeta_n M_n(u, v) },
\end{align*}
where 
\begin{align*}
M_n(u, v) = & \sum_{k=1}^{n_s} {\frac{1}{\| \hat u_n^{[:, k]}\|^{\gamma}}\| u^{[:, k]}\|} + \sum_{k=1}^{n_z} {\frac{1}{\| \hat v_n^{[:, k]}\|^{\gamma}}\|v^{[:, k]}\|},
\end{align*}
$\gamma$ is a positive number, $\hat \alpha_n$ is the ERM estimate, and $\zeta_n$ is an appropriate regularizing constant.
Here, we use the convention $0/0=1$. 

Just as with Group Lasso, only the parameters of the first layer are regularized and they are grouped according to the inputs. 
However, the Adaptive Group Lasso penalty aggressively shrinks the parameters associated with non-significant variables to zero by inflating their penalty weights. 

\paragraph{Group Lasso initialization:} We also consider the following Group Lasso + Adaptive Group Lasso estimator:
\begin{align*}
(\check u_n, &\check v_n, \check w_n, \check b_{1,n}, \check b_{2,n}) := \argmin_{(u,  v, w, b_1, b_2) }{ ~~\frac{1}{n}\sum_{i=1}^n{\ell(\alpha, X_i, Y_i)} + \zeta_n G_n(u, v) },
\end{align*}
where 
\begin{align*}
G_n(u, v) = & \sum_{k=1}^{n_s} {\frac{1}{\| \bar u_n^{[;, k]}\|^{\gamma}}\|u^{[:, k]}\|} + \sum_{k=1}^{n_z} {\frac{1}{\| \bar v_n^{[:, k]}\|^{\gamma}}\|v^{[:, k]}\|},
\end{align*}
$\gamma$ is a positive number, $\bar{\alpha_n}$ is the Group Lasso estimate, and $\zeta_n$ is the regularizing parameter. 

\section{Consistent feature selection via Adaptive Group Lasso}

 As we see in the previous section, an Adaptive Group Lasso estimator inherently depends on the initialization. 
Naturally, the performance of an Adaptive Group Lasso depends strongly on theoretical properties of the initial estimator.  
In this section, we first investigate the convergence of the ERM and the Group Lasso, then use such results to establish feature selection consistency of the Adaptive Group Lasso.

\subsection{Convergence of the ERM and the Group Lasso}

First, we define the risk and empirical risk functions as follows:
\begin{align*}
R(\alpha) &= \mathbb{E}_{(X, Y)\sim P_{X,Y}^*}[(f_{\alpha}(X) - Y)^2] \\
R_n (\alpha) &=\frac{1}{n}\sum_{i=1}^n{(f_{\alpha}(X_i) - Y_i)^2}.
\end{align*}
We note that that the set of all minimizers of the function $R(\alpha)$ is $\mc{H}^*$ (the set of all weight vectors that produce the same input-output map as the generating model). 
We have the following results about the risk functions, for which the proofs can be found in the Appendix. 

\begin{Lemma}[Lipschitzness of the risk functions]
\textbf{}

\begin{itemize}
\item $R(\alpha)$ is a Lipschitz function with Lipschitz constant $c_0 > 0$.
\item For any $\delta > 0$, there exists $M_\delta > c_0$ such that $R_n(\alpha)$ is an $M_\delta$-Lipschitz function with probability at least $1 - \delta$.
\end{itemize}
\label{lem:almostLip}
\end{Lemma}

Next, we provide a generalization bound on the deviation of the empirical risk function from its expected value (proof is in the Appendix).  
\begin{Lemma}[Generalization bound]
For any $\delta>0$, there exist $c_1(\delta)>0$ such that
\[
|R_n( \alpha) - R(\alpha) | \le c_1 \frac{\log n}{\sqrt{n}}, \h \forall \alpha \in \mc{W}.
\]
with probability at least $1-\delta$.
\label{lem:generalization}
\end{Lemma}

We also need the following information bound.
\begin{Lemma}
For any $ \alpha \in \mc{H}^*$, there exist $c_2(\alpha)>0$ and a neighborhood $\mc{U}_\alpha$ such that
\[ 
R(\beta) - R(\alpha)  \ge c_2 \|\beta-\alpha\|^2
\]
for all $ \beta \in \mc{U}_\alpha$.
\label{lem:information}
\end{Lemma}

\begin{proof}
We recall that the Fisher information matrix $I_{\alpha}$ for a Gaussian model is given by
\[
(I_{\alpha})_{ij} = \frac{1}{2 \sigma_e^2} E [ \nabla^2_{\alpha} \ell(\alpha, Y, X) ].
\]
Lemma \ref{lem:characterization} implies that under Assumption $\ref{irr}$, $f_\alpha$ is irreducible for all $\alpha \in \mc{H}^*$.
By Theorem 1 in \citet{fukumizu1996regularity}, the Fisher information matrix $I_{\alpha}$ is positive definite for all $\alpha \in \mc{H}^*$.
We deduce that
\begin{align*}
\ell(\beta, Y, X) - \ell(\alpha, Y, X) &= 2 (f_\alpha(X) - Y)  (\beta - \alpha) \cdot \nabla_{\alpha} f_\alpha(X) \\
&+  \frac{1}{2}(\beta - \alpha)^t \nabla^2_{\alpha} \ell(\alpha, Y, X)  (\beta - \alpha) + o(\|\beta -\alpha\|^2).
\end{align*}
This implies
\[
R(\beta) - R(\alpha) = \sigma^2_e (\beta - \alpha)^t [I_{\alpha}] (\beta - \alpha) + o(\|\beta -\alpha\|^2)
\]
which completes the proof since the Fisher information matrix is positive definite.
\end{proof}

Combining Lemma $\ref{lem:generalization}$ and $\ref{lem:information}$, we have.
\begin{Theorem}[Convergence of ERM]
For any $\delta>0$, there exist $C_\delta>0$ and $N_\delta>0$ such that for all $n \geq N_\delta$, 
\[
\min_{\alpha \in \mc{H}^*} \|\hat \alpha_n -\alpha\| \le \frac{C_\delta\sqrt{\log n}}{n^{1/4} }
\]
with probability at least $1-\delta$.
\label{erm}
\end{Theorem}

\begin{proof}

We define
\[
\mc{U} = \bigcup_{\alpha \in \mc{H}^*}{\mc{U}(\alpha)}
\]
where $\mc{U}(\alpha)$ is defined in Lemma $\ref{lem:information}$. 
Since $\mc{W} \setminus \mc{U}$ is compact and $R(\alpha)$ is a continuous function, there exist $C_{\mc{U}}>0$ such that 
\[
R(\alpha) - R(\alpha^*)  \ge C_{\mc{U}} \h \forall \alpha \not\in \mc{U}.
\]
Thus, for all $n$ such that
\[
2  c_1 \frac{\log n}{\sqrt{n}} \le \frac{C_{\mc{U}}}{2}
\] 
we have
\begin{align*}
0 \leq R(\hat \alpha_n)- R(\alpha^*) &\le R_n(\hat \alpha_n) - R_n(\alpha^*) + 2  c_1 \frac{\log n}{\sqrt{n} } \le 2  c_1 \frac{\log n}{\sqrt{n}} < C_{\mc{U}}
\end{align*}
with probability at least $1-\delta$. 
We deduce that $\hat \alpha_n \in \mc{U}$ with probability at least $1-\delta$.
We have
\begin{align*}
c_2 \|\hat \alpha_n - \alpha\|^2 &\le R(\hat \alpha_n)- R( \alpha) \le 2  c_1 \frac{\log n}{\sqrt{n}}
\end{align*}
which completes the proof.
\end{proof}

Similarly, we also have the following theorem for Group Lasso.
\begin{Theorem}[Convergence of Group Lasso]
Assuming that $\lambda_n \to 0$.
For any $\delta>0$, there exist $C_\delta>0$, $N_\delta>0$ such that for all $n \ge N_\delta$, 
\begin{align*}
\min_{\alpha \in \mc{H}^*} \|\bar \alpha_n -\alpha\| &\le C_\delta \left( \frac{\log n}{\sqrt{n} } + \lambda_n^2\right)^{1/2}
\end{align*}
with probability at least $1-\delta$.
\label{group}
\end{Theorem}

\begin{proof}
We have
\[
R_n( \bar \alpha_n) + \lambda_n L(\bar \alpha_n)  \le R_n( \alpha^*) + \lambda_n L(\alpha^*) 
\]
which implies 
\begin{align*}
R(\bar \alpha_n)- R( \alpha^*) &\le R_n(\bar \alpha_n) - R_n(\alpha^*) + 2  c_1 \frac{\log n}{\sqrt{n}} \\
&\le 2  c_1 \frac{\log n}{\sqrt{n}} + \lambda_n \left( L(\alpha^*) - L(\bar \alpha_n)\right)
\end{align*}
since $\mc{W}$ is bounded. 
Using the same argument as in the proof of Theorem $\ref{erm}$, we conclude that when $n$ is large enough, $\hat \alpha_n \in \mc{U}(\alpha)$ for some $\alpha \in \mc{H}^*$.
Since $L(\alpha)$ is a Lipschitz function, we have
\begin{align*}
c_2 \|\bar  \alpha_n - \alpha\|^2  &\le 2  c_1 \frac{\log n}{\sqrt{n}} + \lambda_n \left( L(\alpha) - L(\bar \alpha_n)\right)\\
&\le 2  c_1 \frac{\log n}{\sqrt{n}} + \lambda_n C\|\bar \alpha_n - \alpha\|\\
&\le 2  c_1 \frac{\log n}{\sqrt{n}} + \frac{c_2}{2}\|\bar \alpha_n - \alpha\|^2 + \frac{C^2\lambda_n^2}{2 c_2}
\end{align*}
which completes the proof. 
\end{proof}


\subsection{Feature selection consistency of the adaptive group lasso}

\begin{Theorem}[Convergence of ERM + Adaptive Group Lasso]
Assuming that $\zeta_n \to 0$, for all $\delta>0$, there exists $C_{\delta}, N_{\delta}>0$ such that for all $n \ge N_{\delta}$, we have
\[
\min_{\alpha \in \mc{H}^*} \|\tilde \alpha_n -\alpha\|  \le C_{\delta}\left( \frac{\log n}{\sqrt{n}} + \zeta_n \right)^{1/2}
\]
\label{erm-agl}
with probability at least $1-\delta$. 
\end{Theorem}

\begin{proof}
By Theorem $\ref{erm}$, for $n$ large enough,
\[
\min_{\alpha \in \mc{H}^*} \|\hat \alpha_n -\alpha\| \le \frac{C_\delta\sqrt{\log n}}{n^{1/4} }
\]
with probability at least $1-\delta$.
Therefore, 
\[
\min_{\alpha \in \mc{H}^*} \|\hat u_n^{[:, k]} - u^{[:, k]}\| \le \frac{C_\delta\sqrt{\log n}}{n^{1/4} }, \h \forall k =1, \ldots, n_s.
\]
By Lemma $\ref{lem:characterization}$, we conclude that $\hat u_n^{[:, k]}$ is bounded away from zero as $n \to \infty$. 
We deduce that for $\alpha \in \mc{H}^*$, 
\[
M_n(\alpha) = \sum_{i=1}^{n_s} \frac{1}{\| u_n^{[:, k]}\|^{\gamma}} \|u^{[:, k]}\|  < \infty.
\]
Thus,
\begin{align*}
R(\tilde\alpha_n)- R( \alpha^*) &\le R_n(\tilde \alpha_n) - R_n(\alpha^*) + 2  c_1 \frac{\log n}{\sqrt{n}} \\
&\le 2  c_1 \frac{\log n}{\sqrt{n}} + \zeta_n \left( M_n(\alpha^*) - M_n(\bar \alpha_n)\right)\\
&\le 2  c_1 \frac{\log n}{\sqrt{n}} + \zeta_n M_n(\alpha^*).
\end{align*}
Since $M_n(\alpha)$ is bounded, using the same argument as in the proof of Theorem $\ref{erm}$, we conclude that when $n$ is large enough, $\tilde \alpha_n \in \mc{U}(\alpha)$ for some $\alpha \in \mc{H}^*$ and there exists $\alpha \in \mc{H}^*$
\[
c_2 \|\bar  \alpha_n - \alpha\|^2 \le R(\tilde\alpha_n)- R( \alpha^*) \le 2  c_1 \frac{\log n}{\sqrt{n}} + \zeta_n M_n(\alpha^*)
\]
which completes the proof.
\end{proof}

We are now ready to prove the main theorem of our paper. 

\begin{Theorem}[Feature selection consistency of ERM + Adaptive Group Lasso]
For $\gamma>0$, $\mu \in (0, \gamma/4)$ and  $\zeta_n =\Omega (n^{-\gamma/4 + \mu})$, then the ERM + Adaptive Group Lasso is consistent for feature selection. 
That is, for any $\delta >0$, there exists $N_\delta$ such that for $n > N_\delta$, $\tilde u_n^{[:, k]} \ne 0, ~ \forall k =1, \ldots, n_s$, and $\tilde v_n^{[:, k]} = 0, ~ \forall k =1, \ldots, n_z$ with probability at least $1 -\delta$.
\end{Theorem}

\begin{proof}
Theorem $\ref{erm-agl}$ provides that
\[
\min_{\alpha \in \mc{H}^*} \|\tilde \alpha_n -\alpha\|  \le C_{\delta}\left( \frac{\log n}{\sqrt{n}} + \zeta_n \right)^{1/2}
\]
By Lemma $\ref{lem:characterization}$, we conclude that $\tilde u_n^{[:, k]}$ is different from zero for $n$ large enough. 
Next, we will prove that $\tilde v_n =0$.
Note that from Theorem $\ref{erm}$
\[
\|\hat v_n^{[:, k]}\| \le  \frac{C_\delta \sqrt{\log n}}{n^{1/4} }
\]
with probability at least $1-\delta$.
Thus, 
\begin{equation}
\lim_{n \to \infty}{~\zeta_n \frac{1}{\|\hat v_n^{[:, k]}\|^{\gamma}}} \ge C_\delta^{-\gamma} \lim_{n \to \infty}{~\zeta_n \frac{n^{\gamma/4}}{(\log n)^{\gamma/2}}} = \infty.
\label{bound}
\end{equation}
Now, we assume that $\tilde v_n^{[:, k]} \ne 0$ for some $k$ and define a new weight configuration $g_n$ obtained from $\tilde \alpha_n$ by setting the $v^{[:, k]}$ component to $0$.
By definition of the estimator $\tilde \alpha_n$, we have
\[
R_n( \tilde \alpha_n) +  \zeta_n   \frac{1}{\|\hat v_n^{[:, k]}\|^{\gamma}} \|\tilde v_n^{[:, k]}\|\le R_n(g_n).
\]
By Lemma $\ref{lem:almostLip}$, we have
\begin{align*}
 \zeta_n   \frac{1}{\|\hat v_n^{[:, k]}\|^{\gamma}} \|\tilde v_n^{[:, k]}| &\le  R_n(g_n) - R_n( \tilde \alpha_n)\\
 & \le c_0 \|g_n - \tilde \alpha_n\|  = M_{\delta} \|\tilde v_n^{[:, k]}\|.
\end{align*}
with probability at least $1-\delta$. 
Since $\tilde v_n^{[:, k]} \ne 0$, we deduce that
\[
\zeta_n   \frac{1}{\|\hat v_n^{[:, k]}\|^{\gamma}}  \le M_{\delta},
\]
which contradicts $\eqref{bound}$.
This completes the proof. 
\end{proof}

\begin{Theorem}[Feature selection consistency of Group Lasso + Adaptive Group Lasso]
For $\gamma>0$,  $\mu \in (0, \gamma/4)$, $\zeta_n =\Omega (n^{-\gamma/4 + \mu})$ and $\zeta_n = \Omega(\lambda_n^{\gamma+\mu})$ where $\lambda_n$ is the regularizing constant of the initial Group Lasso estimate. 
Then, the Group Lasso + Adaptive Group Lasso is consistent for feature selection. 
\end{Theorem}

\begin{proof}
The proof is similar to that of the ERM + Adaptive Group Lasso, with the only exception that 
\[
\|\bar v_n^{[:, k]}  \| \le C \left( \frac{\log n}{\sqrt{n} } + \lambda_n^2 \right)^{1/2}.
\]
with probability $1-\delta$. 
\end{proof}

\section{Simulation studies}

We use synthetic data to investigate the performance of our feature selection procedure. 
Throughout the experiment, we consider single-output feed-forward neural networks with one hidden layer of ten hidden nodes. 
Two separate sets of experiments are considered. 
In the first set, one of the two input variables is non-significant while the other is significant.
In the second set of simulations, two of the five input variables are non-significant and the other three are. 

In each set of experiments, we simulate datasets of size $n=1000$ from the generative model 
\[
Y= f_{u^*, 0, w^*, b_1^*, b_2^*}(X) + \epsilon
\]
where the noise $\epsilon^{[i]}$ is sampled independently from $\mathcal{N}(0, \sigma^2_e)$ with various values of $\sigma_e^2$. 
Each component of $X, b_1^*, b_2^*$ is sampled independently from $\mathcal{N}(0,1)$, and each component of $u^*, w^*$ is independently sampled from $ \mathcal{N}(1,1)$. 

\begin{figure}[ht]
\centering
    \includegraphics[width=0.47\textwidth]{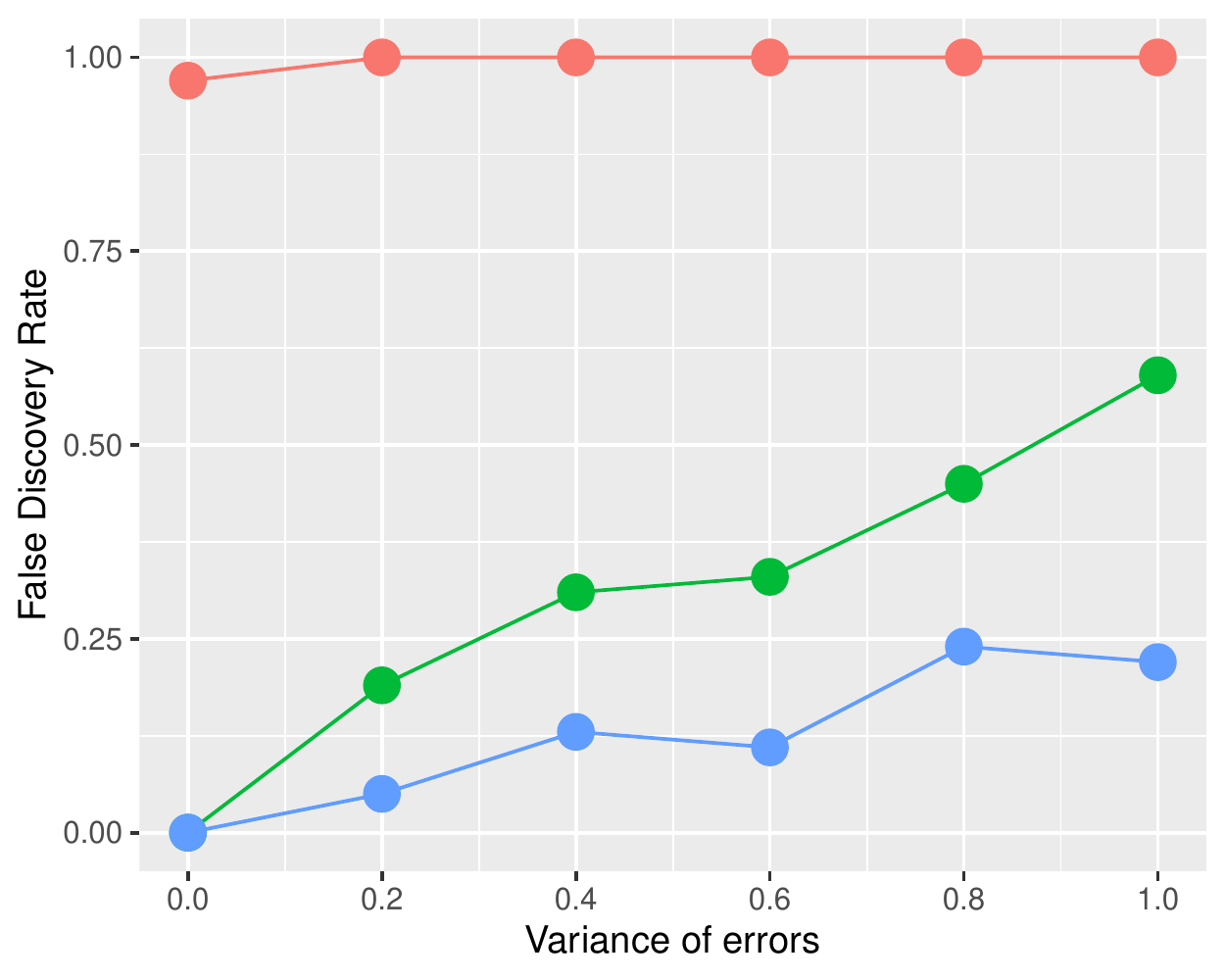}
 \includegraphics[width=0.47\textwidth]{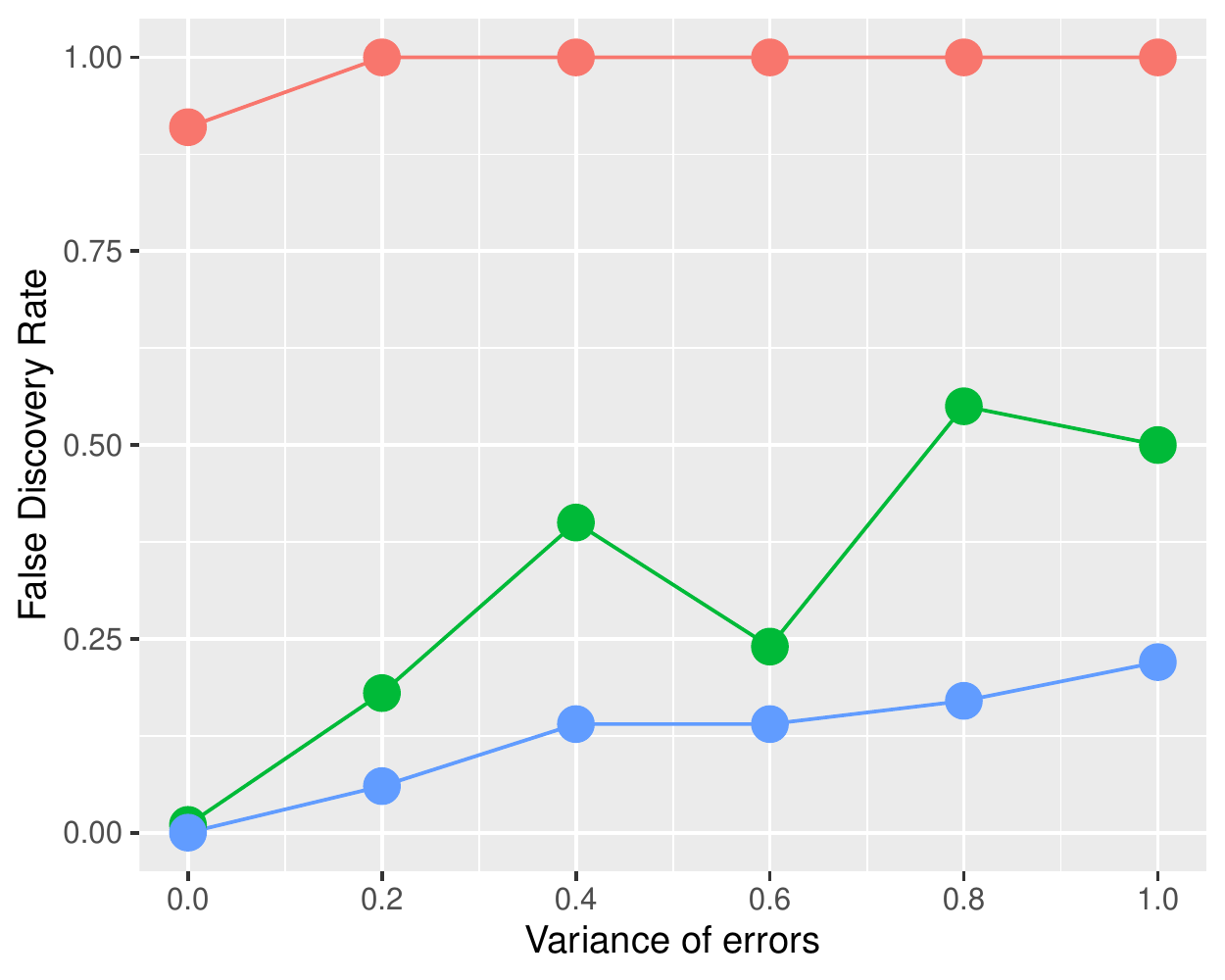}
 \caption{False Discovery Rates (FDRs) of the two non-significant variables (each subfigure corresponds to a non-significant variable) in the second set of experiment with varying degrees of noise by the three feature selection algorithms:  the Group Lasso (red curves), the ERM + Adaptive Group Lasso (green), and Group Lasso + Adaptive Group Lasso (blue).}
\label{fig:5var}
\end{figure}

We are interested in the performance of our feature selection procedures with various noise levels. 
For each $\sigma_e^2$ in $\{0, 0.2, 0.4, 0.6, 0.8, 1\}$, we simulate 100 datasets of size $n=1000$. 
We apply three methods of feature selection (the Group Lasso, the ERM + Adaptive Group Lasso, and Group Lasso + Adaptive Group Lasso) on each simulated dataset with $\gamma = 2$.
The regularizing constant is chosen from the set $\{0.001, 0.01, 0.1, 1\}$ using three-fold cross-validation. 
The algorithms are trained using Adam optimizer over 10000 epochs with batch-size of 200, of which the codes are created using Python package \emph{lasagne} \cite{lasagne} and Python library \emph{theano} \cite{team2016theano}. 
Following the convention used in \citet{scardapane2017group}, we deselect a variable if the $\ell_2$-norm of the groups of parameters associated with that variable is below the cut-off value $10^{-3}$.

The result is presented in Figure \ref{fig:5var}, where the False Discovery Rate (FDR) of the two non-significant variables in the second set of experiments (with 5 input variables) for each of the algorithms are reported. 
A similar result for the first set of experiments is also provided in the Appendix. 
In both sets of experiments, we observe that across all algorithms, the FDR tends to increase as the variance of the error increases. 
The Group Lasso + Adaptive Group Lasso perform the best, while the Group Lasso fails to recognize the non-significant variables when there are errors in the model ($\sigma^2_e >0$). 
The result indicates that for highly non-linear models such as neural networks, standard Lasso-type algorithms may not be aggressive enough to enforce sparsity and data-dependent approaches such as Adaptive Lasso might be necessary for effective feature selection.
On the other hand, it is worth noting that all three methods can detect the significant variables efficiently.

\section{Data analysis}

\begin{figure*}
\centering
    \includegraphics[width=1\textwidth]{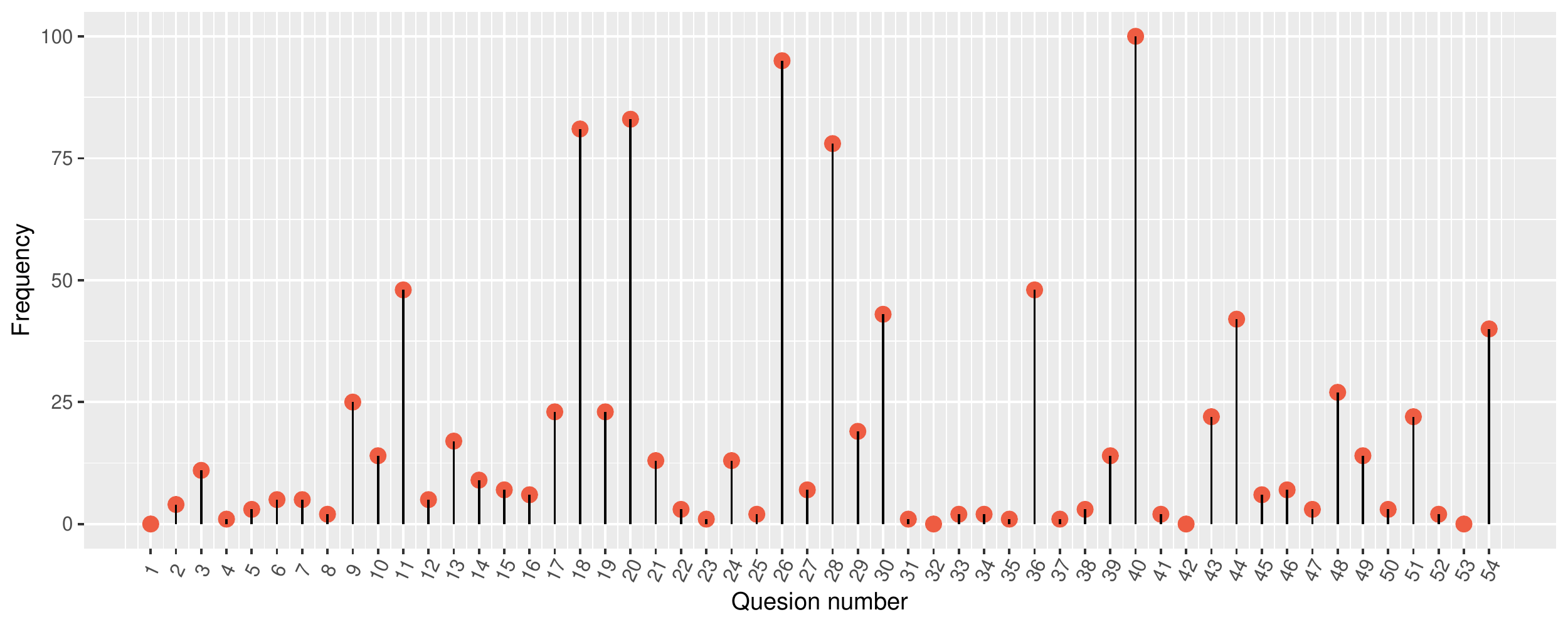}
\caption{Frequency of being selected by Group Lasso + Adaptive Group Lasso of each feature, computed out of 100 runs of the algorithm.}
\label{fig:analysis}
\end{figure*}

In addition to the simulation examples presented in the previous section, we also demonstrate the practical use of our feature selection approach to analyze the Divorce Prediction dataset \citep{yontem2019divorce} from the UCI Machine Learning Repository \cite{Dua:2019}.
The dataset consists of survey data from 170 married couples, where each participant rates 54 statements about their marriage on the scale from 0--4 (a detailed list of the statements are provided in the Appendix). 
The dataset was collected from seven different regions of Turkey (but predominantly from the Black Sea region) through face-to-face interviews and via Google Drive. 
Among the participants, 49\% were divorced and 51\% were married couples. 

Based on the success of its performance in simulation studies, we apply the Group Lasso + Adaptive Group Lasso to analyze the data set. 
In this analysis, we also use the feed-forward network structure with one hidden layer and the hyperbolic tangent activation function. 
Since the data is binary, we further add a sigmoid activation function to the output layer of the network. 
The model is then trained using Adam optimizer with binary cross-entropy loss over 10000 epochs with batch-size of 32. 
The 3-fold cross-validation is employed to select the regularizing parameter $\lambda$ from the set $\{0.001, 0.01, 0.1, 1, 2, 4, 8, 16 \}$. 
To take into account the stochasticity of the selection procedure, we run the algorithm 100 times and record the number of times each feature is chosen by the algorithm.
The result is presented in Figure \ref{fig:analysis}, which shows that the five most important variables are selected more than 75 times while all other variables appear fewer than 50 times on the selections. 

To validate the result, we fit the dataset on the full set of variables and on the selected set of five variables using ERM algorithm with the same neural network structure and compare their performances.
In each iteration, the dataset is split into a training set and a test set by a ratio of 75-25. 
We use ERM to train on the training set with the binary cross-entropy loss and compute the accuracy on the test set. 
The process is repeated 100 times and the average testing accuracy for the full model and the selected model are 97.7\% and 98.1\%, respectively.
The result indicates that our feature selection procedure can select the most important features for prediction.

\section*{Discussions and future work}

In this work, we propose the use of the Adaptive Group Lasso for selecting important features of neural networks.
Under mild regularity conditions, we establish feature selection consistency of two versions of the Adaptive Group Lasso on three-layer neural networks with hyperbolic tangent activation function. 
We have demonstrated the efficiency and effectiveness of our methods with both synthetic and real datasets. 
We show that the standard Group Lasso estimator has difficulty identifying the input support while the Adaptive Group Lasso can accurately and consistently select the correct set of significant variables. 

There are several avenues for future directions of this work. 
Firstly, \citet{sussmann1992uniqueness} and \citet{fukumizu1996regularity} outlined conditions of the activation functions for which characterization of functional equivalent neural networks can be obtained.
This might be used to extend our results to other activation functions, including the sigmoid and the leaky-ReLU functions. 
Secondly, the possibility of extending the results to deeper networks relies mainly on re-establishing Lemma $\ref{lem:characterization}$, which dictates that: \emph{there is no alternative weight vector with different input support but has the same input-output mapping as the generating model}. 
We note that if this condition fails, then the feature selection problem is ill-defined and it is not possible to do feature selection in this scenario.
Further understanding about the landscape of deep neural networks will provide additional insights about when feature selection is possible and how our method can be adapted to tackle the problem. 

\section*{Acknowledgement}

LSTH was supported by startup funds from Dalhousie University, the Canada Research Chairs program, the NSERC Discovery Grant RGPIN-2018-05447, and the NSERC Discovery Launch Supplement DGECR-2018-00181.

\clearpage

\bibliographystyle{chicago}
\bibliography{alg}

\clearpage

\onecolumn

\appendix
\section{Proofs}
\subsection{Lemma \ref{lem:almostLip}}

Since the hyperbolic tangent function is a bounded Lipschitz function and the weight space $\mc{W}$ is bounded, $f_{\alpha}(X)$ is also bounded.
Therefore, 
\begin{align*}
|R(\alpha) - R(\alpha')| &\leq C_1 \| \alpha - \alpha' \| \cdot \mathbb{E}_{(X, Y)\sim P_{X,Y}^*}| f_{\alpha}(X) + f_{\alpha'}(X) - 2Y | \\
&\leq C_2 \| \alpha - \alpha' \| (C_2 + 2\mathbb{E}_{(X, Y)\sim P_{X,Y}^*}| Y - f_{\alpha^*}(X) | ) \\
&\leq c_0 \| \alpha - \alpha' \|.
\end{align*}
Similarly,
\begin{align*}
| R_n(\alpha) - R_n(\alpha') | & \leq C_1 \| \alpha - \alpha' \| \left ( C_2 + \frac{2}{n}\sum_{i=1}^n{| Y_i - f_{\alpha^*}(X_i) |} \right ) \\
&= C_1 \| \alpha - \alpha' \| \left ( C_2 + \frac{2}{n}\sum_{i=1}^n{| \epsilon_i |} \right ).
\end{align*}
The proof is completed by noting
\[
P \left (\frac{1}{n}\sum_{i=1}^n{| \epsilon_i |} > M \right ) \leq \frac{E | \epsilon_1 |}{M}
\]
that for any $M > 0$.

\subsection{Lemma \ref{lem:generalization}}

Note that $n R_n( \alpha)/\sigma^2_e$ follows a  non-central chi-squared distribution with $n$ degrees of freedom and $f_\alpha(X)$ is bounded.
By applying Theorem 7 in \citet{zhang2018}, we have
\begin{align*}
& \mathbb{P}\left[ | R_n( \alpha) - R(\alpha) | > t/2 \right] \\
& \leq  2\exp \left (- \frac{C_1 n^2 t^2}{n + 2 \sum_{i=1}^n{[f_\alpha(X) - f_{\alpha^*}(X)]^2} } \right ) \\
& \leq 2\exp(- C_2 n t^2),
\end{align*}
for all
\[
0 < t < \frac{n + \sum_{i=1}^n{[f_\alpha(X) - f_{\alpha^*}(X)]^2}}{n}.
\]
We define the events
\[
\mc{A}(\alpha, t) = \{|R_n( \alpha) - R(\alpha) | > t/2 \},
\]
\begin{align*}
\mc{B}(\alpha, t) = \{&\exists \alpha' \in \mc{W}~\text{such that}~\\
&\|\alpha'-\alpha\| \le \frac{t}{4M_\delta}~ \text{and}~ |R_n( \alpha') - R(\alpha') | > t \},
\end{align*}
and
\[
\mc{C} = \{ |R_n( \alpha) - R_n( \alpha')| \leq M_\delta \| \alpha - \alpha' \|, \forall \alpha, \alpha' \in \mc{W} \}.
\]
Here, $M_\delta$ is defined in Lemma \ref{lem:almostLip}.
By Lemma \ref{lem:almostLip}, $\mc{B}(\alpha, t) \cap \mc{C} \subset \mc{A}(\alpha, t)$ and $P(\mc{C}) \geq 1 - \delta$.

Let $m=dim(\mc{W})$, there exist $C_3(m) \ge 1$ and a finite set $\mathcal{H} \subset \mc{W}$ such that
\[
\mc{W} \subset \bigcup_{\alpha \in \mathcal{H}}{\mc{V}(\alpha, \epsilon)} \h \text{and}\h  |\mathcal{H}| \le C_3 /\epsilon^{m}
\]
where $\epsilon=t/(4M_\delta)$, $\mc{V}(\alpha, \epsilon)$ denotes the open ball centered at $\alpha$ with radius $\epsilon$, and $|\mathcal{H}|$ denotes the cardinality of $\mathcal{H}$.
By a union bound, we have
\[
\mathbb{P}\left[ \exists \alpha \in \mathcal{H}: \left |R_n( \alpha) - R(\alpha)\right |> t/2\right]  \le 2 \frac{C_3 (4M_\delta)^m}{t^m}e^{- C_2 n t^2}.
\]
Using the fact that $\mc{B}(\alpha, t) \cap \mc{C} \subset \mc{A}(\alpha, t), ~\forall \alpha \in \mathcal{H}$, we deduce
\[
\mathbb{P}\left[ \{ \exists \alpha  \in \mc{W}: \left |R_n( \alpha) - R(\alpha)\right |> t \right \} \cap \mc{C}]  \le C_4 t^{-m}e^{- C_2 n t^2}.
\]
Hence,
\[
\mathbb{P}\left[ \{ \exists \alpha  \in \mc{W}: \left |R_n( \alpha) - R(\alpha)\right |> t \right \}]  \le C_4 t^{-m}e^{- C_2 n t^2} + \delta.
\]

To complete the proof, we chose $t$ in such a way that $C_4 t^{-m}e^{- C_2 n t^2} ~ \le~ \delta$.
This can be done by choosing $t = \mathcal{O}(\log n /\sqrt{n})$. 

\section{Simulation studies}

\begin{figure}[ht]
\centering
    \includegraphics[width=0.47\textwidth]{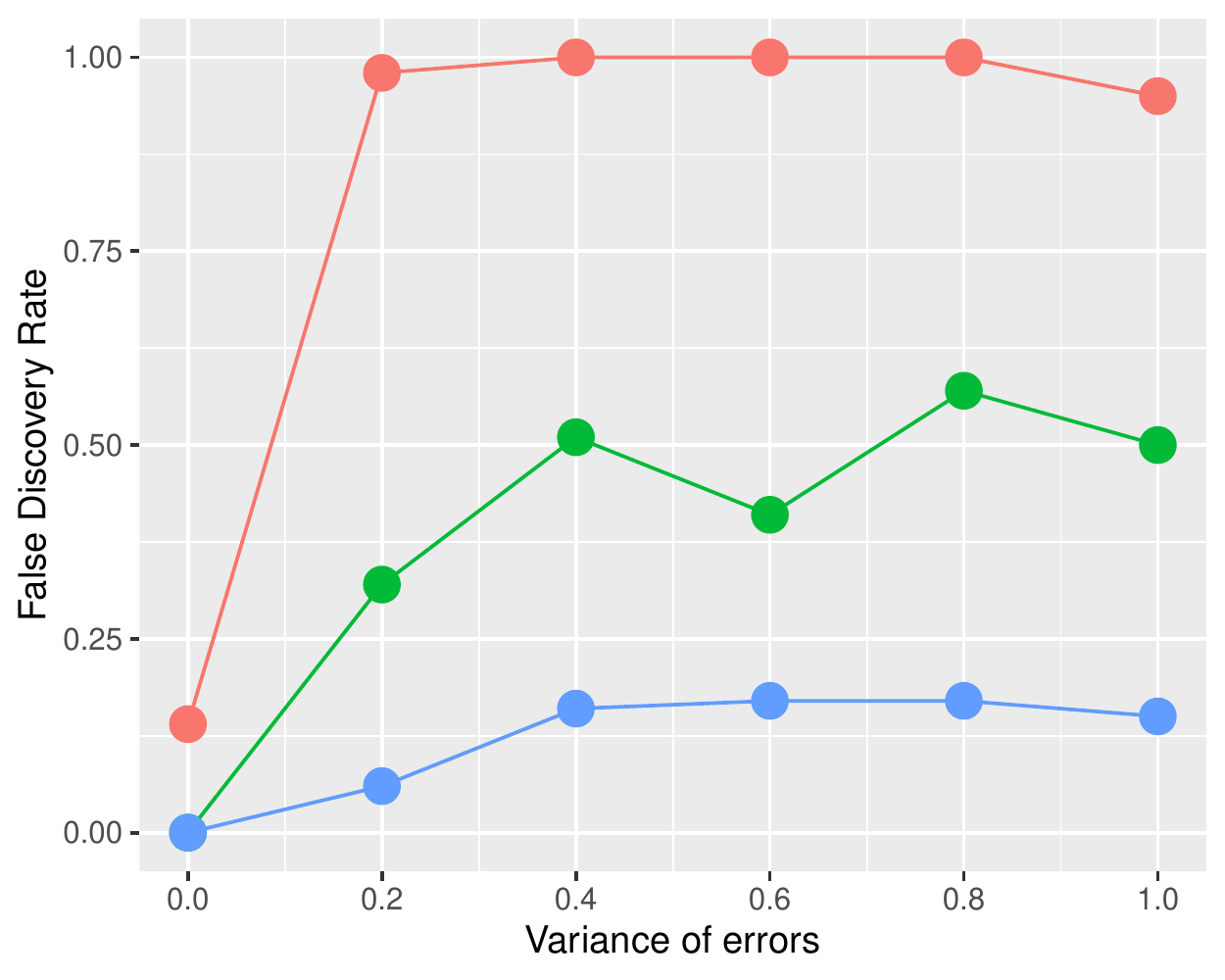}
 \caption{False Discovery Rates (FDRs) of the non-significant variable in the first set of experiment with varying degrees of noise by the three feature selection algorithms:  the Group Lasso (red curves), the ERM + Adaptive Group Lasso (green), and Group Lasso + Adaptive Group Lasso (blue).}
\label{fig:2var}
\end{figure}

\section{Divorce Predictors Dataset }

This dataset is available on the UCI Machine Learning Repository (\url{https://archive.ics.uci.edu/ml/datasets/Divorce+Predictors+data+set}.) The list of $54$ statements of the data are
\begin{enumerate}
\item If one of us apologizes when our discussion deteriorates, the discussion ends. 
\item I know we can ignore our differences, even if things get hard sometimes. 
\item When we need it, we can take our discussions with my spouse from the beginning and correct it. 
\item When I discuss with my spouse, to contact him will eventually work. 
\item The time I spent with my wife is special for us. 
\item We don't have time at home as partners. 
\item We are like two strangers who share the same environment at home rather than family. 
\item I enjoy our holidays with my wife. 
\item I enjoy traveling with my wife. 
\item Most of our goals are common to my spouse. 
\item I think that one day in the future, when I look back, I see that my spouse and I have been in harmony with each other. 
\item My spouse and I have similar values in terms of personal freedom. 
\item My spouse and I have similar sense of entertainment. 
\item Most of our goals for people (children, friends, etc.) are the same. 
\item Our dreams with my spouse are similar and harmonious. 
\item We're compatible with my spouse about what love should be. 
\item We share the same views about being happy in our life with my spouse. 
\item My spouse and I have similar ideas about how marriage should be. 
\item My spouse and I have similar ideas about how roles should be in marriage. 
\item My spouse and I have similar values in trust. 
\item I know exactly what my wife likes. 
\item I know how my spouse wants to be taken care of when she/he sick. 
\item I know my spouse's favorite food. 
\item I can tell you what kind of stress my spouse is facing in her/his life. 
\item I have knowledge of my spouse's inner world. 
\item I know my spouse's basic anxieties. 
\item I know what my spouse's current sources of stress are. 
\item I know my spouse's hopes and wishes. 
\item I know my spouse very well. 
\item I know my spouse's friends and their social relationships. 
\item I feel aggressive when I argue with my spouse. 
\item When discussing with my spouse, I usually use expressions such as `you always' or `you never'. 
\item I can use negative statements about my spouse's personality during our discussions. 
\item I can use offensive expressions during our discussions. 
\item I can insult my spouse during our discussions. 
\item I can be humiliating when we discussions. 
\item My discussion with my spouse is not calm. 
\item I hate my spouse's way of open a subject. 
\item Our discussions often occur suddenly. 
\item We're just starting a discussion before I know what's going on. 
\item When I talk to my spouse about something, my calm suddenly breaks. 
\item When I argue with my spouse, I only go out and I don't say a word. 
\item I mostly stay silent to calm the environment a little bit. 
\item Sometimes I think it's good for me to leave home for a while. 
\item I'd rather stay silent than discuss with my spouse. 
\item Even if I'm right in the discussion, I stay silent to hurt my spouse. 
\item When I discuss with my spouse, I stay silent because I am afraid of not being able to control my anger. 
\item I feel right in our discussions. 
\item I have nothing to do with what I've been accused of. 
\item I'm not actually the one who's guilty about what I'm accused of. 
\item I'm not the one who's wrong about problems at home. 
\item I wouldn't hesitate to tell my spouse about her/his inadequacy. 
\item When I discuss, I remind my spouse of her/his inadequacy. 
\item I'm not afraid to tell my spouse about her/his incompetence. 
\end{enumerate}

\end{document}